\documentclass[conference]{IEEEtran} 

\IEEEoverridecommandlockouts                              

\usepackage{cite}
\usepackage{amsmath,amssymb,amsfonts}
\usepackage{algorithmic}
\usepackage{graphicx}
\usepackage{textcomp}
\usepackage{xcolor}
\usepackage{siunitx} 
\usepackage{amsthm}
\usepackage{mathtools}
\usepackage{cleveref}
\usepackage{placeins}
\usepackage{dblfloatfix}
\usepackage[acronym]{glossaries}
\usepackage{tikz}
\usepackage[normalem]{ulem}

\graphicspath{{./pictures/}}

\newcommand{\T}[1]{\boldsymbol{{#1}}}
\newcommand{\inv}[1]{#1^{-1}}

\newcommand{\q}{\T{q}}
\newcommand{\dq}{\dot{\T{q}}}
\newcommand{\ddq}{\ddot{\T{q}}}

\newcommand{\M}{\T{M}}

\newcommand{\Real}{\mathbb{R}}

\newcommand{\qerr}{\T{e}_q}
\newcommand{\dqerr}{\dot{\T{e}}_q}

\newcommand{\dqr}{\dot{\mathbf{q}}_r}

\newcommand{\etheta}{\T{e}_\theta}

\newcommand{\ql}{\q}

\newcommand{\qm}{\T \theta}

\newcommand{\ddqm}{\ddot{\T \theta}}
\newcommand{\ctrlinput}{\T u}

\newcommand{\springtorque}{\T \psi}

\newcommand{\klipschitzmax}{k_{\psi}}
\newcommand{\ksmin}{k_{s,min}}
\newcommand{\Ks}{\T K_s}
\newcommand{\taue}{\T \tau_{e}}
\newcommand{\abs}[1]{{\vert\vert #1 \vert \vert}_2}

\newcommand{\khat}{\hat{\T k}}
\newcommand{\kerr}{\T{e}_k}
\newcommand{\dkerr}{\dot{\T{e}}_k}
\newcommand{\Ltwo}{\mathcal{L}_2}

\newtheorem{theorem}{Theorem}

\newtheorem{assumption}{Assumption}

\setacronymstyle{long-short}
\newacronym{CoM}{CoM}{center of mass}

\usepackage{setspace}
\definecolor{sidebox_bg}{rgb}{1,1,0.4}

\title{\vspace{8mm} Adaptive Control of Motor-Position-Controlled Flexible Joint Robots with Uncertain Joint Stiffness
}

\author{\IEEEauthorblockN{Annika Kirner}
	\IEEEauthorblockA{\textit{Automation and Control Institute} \\
		\textit{TU Wien}\\
		Vienna, Austria \\
		kirner@acin.tuwien.ac.at}
	\and
		\IEEEauthorblockN{Grazia Zambella}
	\IEEEauthorblockA{\textit{Automation and Control Institute} \\
		\textit{TU Wien}\\
		Vienna, Austria \\
		zambella@acin.tuwien.ac.at}
	\and
		\IEEEauthorblockN{Igor Kova\v{c}evi\'{c}}
	\IEEEauthorblockA{\textit{Automation and Control Institute} \\
		\textit{TU Wien}\\
		Vienna, Austria}
		\and
	\IEEEauthorblockN{Hannes H\"oppner}
	\IEEEauthorblockA{\textit{Soft Interactive Robotics Laboratory} \\
		\textit{Berliner Hochschule f\"ur Technik}\\
		Berlin, Germany \\
		hannes.hoeppner@bht-berlin.de}
	\and
	\IEEEauthorblockN{Jee-Hwan Ryu}
	\IEEEauthorblockA{\textit{Dept. of Civil and Environmental Engineering} \\
		\textit{Korea Advanced Institute of Science \& Technology}\\
		Seoul, South Korea \\
		jhryu@kaist.ac.kr}

	\and
	\IEEEauthorblockN{Christian Ott$^1$}
	\IEEEauthorblockA{\textit{Automation and Control Institute} \\
		\textit{TU Wien}\\
		Vienna, Austria \\
		christian.ott@tuwien.ac.at}
	
	\thanks{This work was supported in part by the European Research Council (ERC) through the European Union’s Horizon Europe research and innovation programme under Grant 101248099 (CORIM).}
	
	\thanks{$^1$Christian Ott is also with the Institute of Robotics and Mechatronics, German Aerospace Center (DLR), Wessling, Germany.}
}

\begin{document}

\maketitle
\thispagestyle{empty}
\pagestyle{empty}

\begin{abstract}
Model-based control of flexible joint robots with position-controlled actuators relies on accurate knowledge of the joint compliance. In practice, precise stiffness models are often unavailable as the properties of physical elastic elements vary with operating conditions and slowly change over time due to wear and aging. To improve model-based control of these systems, we propose an adaptive control approach in this work, which updates an estimate of the uncertain, nonlinear torque-deflection relation of each joint. As opposed to classical adaptive control approaches for non-elastic robots, we rely on an implicit control law and a control-input-dependent regressor matrix to account for the uncertain joint stiffness. 
We analyze robustness of the approach against errors induced by the motor position controller.
Experimental results on a flexible joint with nonlinear stiffness characteristics demonstrate the effectiveness of the proposed approach.

\end{abstract}


\section{Introduction}
Common model-based control approaches for flexible joint robots require accurate models of the joint stiffness \cite{DeLuca2008, Spong1987}. In practice, however, the torque-deflection characteristics of physical elastic elements are subject to gradual variations due to effects such as material aging, temperature changes, and load-dependent hysteresis. Explicitly modeling these effects is challenging and often impractical. As a result, torque-deflection models identified from calibration measurements may not remain accurate over time or across operating conditions.

This issue can degrade the achievable control performance of flexible joint robots, also if their elastic actuators employ servo motors. Examples for such actuators are the variable stiffness actuators (VSA) introduced in \cite{Catalano2011}. They have been designed to be low-cost and easily deployable while being inherently compliant. The actuators \cite{Catalano2011} have been used extensively in test benches for validating control strategies \cite{DellaSantina2017, Moyron2025, Pierallini2020}. Further, they have been successfully incorporated into more complex robotic platforms such as the wheeled humanoid robot Alter-Ego \cite{Lentini2019} demonstrating their potential for being used in daily tasks involving physical interaction, such as opening a door \cite{Zambella2025a}.

\begin{figure}[tb]
	\centering
	\includegraphics[width=0.49\textwidth]{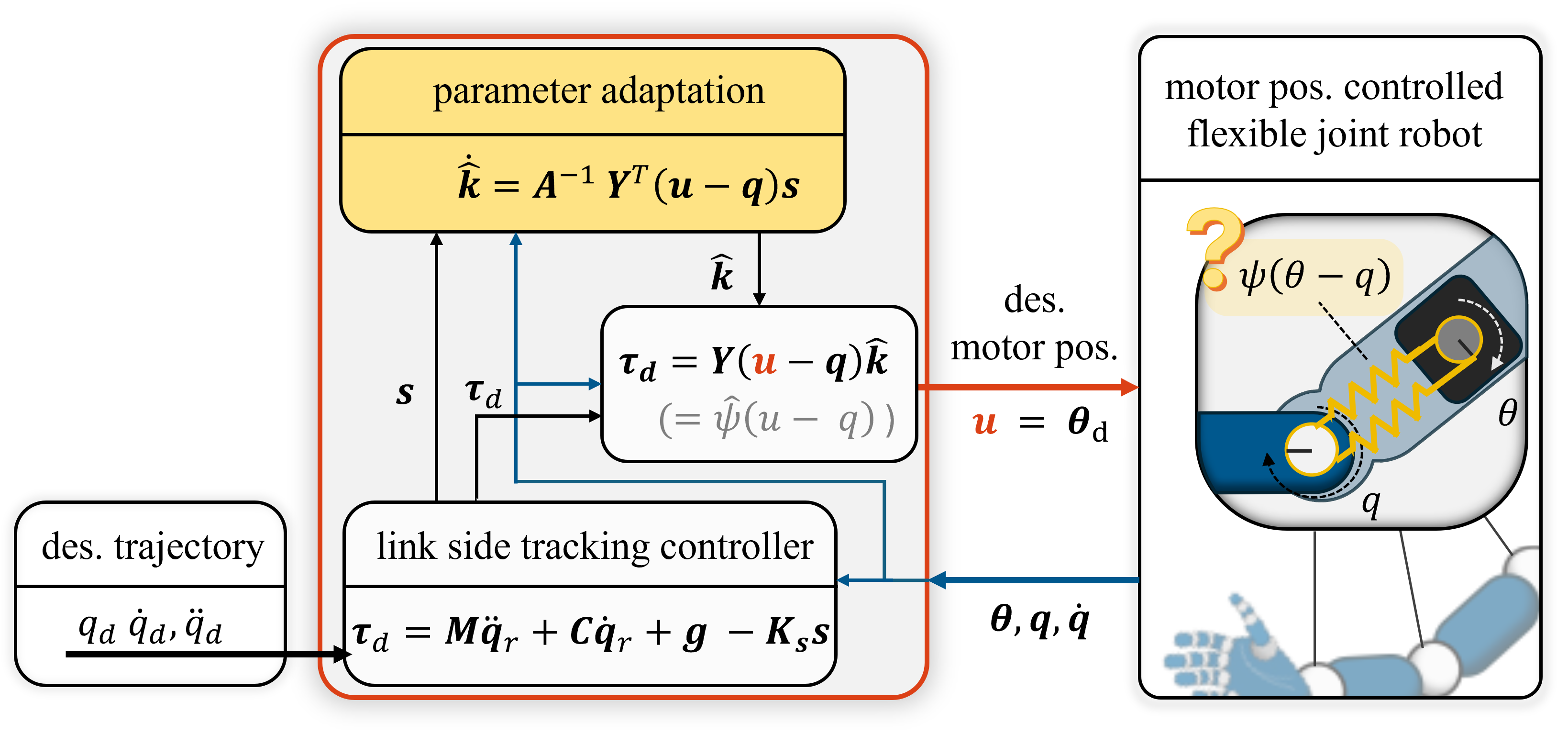}
	\caption{Blockdiagram of the proposed adaptive link-side tracking controller for motor position controlled flexible joint robots. An estimate of the uncertain, nonlinear torque-deflection relation $\springtorque(\T\theta- \q)$ is adapted online such that the link-side tracking errors converge.}
	\label{fig:intro}
\end{figure}

Flexible joint robots equipped with servo motors provide a motor-position interface instead of a torque interface. Model-based link-side control approaches for these robots usually determine the desired motor position to be commanded from the desired torque to be injected by inverting the torque-deflection model \cite{Moyron2025, Lentini2019}. Consequently, the accuracy of the torque generation, and thus the overall control performance, strongly depends on the fidelity of the utilized model. This motivates the development of control strategies for motor-position-controlled flexible joint robots that enable accurate link-side control despite uncertainties in the joint stiffness characteristics.

A common way to address uncertain system parameters is adaptive control. Adaptive controllers have been widely studied and successfully used to compensate for parametric uncertainties in inertial parameter by adjusting estimates online \cite{Slotine1987, Slotine1989, Lee2018}. Extensions of these ideas to flexible-joint robots with uncertain inertial parameters have also been proposed \cite{Spong1989, Ghorbel1989, Lozano-Leal1991}. Besides imprecise knowledge of inertial parameters, also uncertain stiffness is treated in \cite{Tian1995}. However, these approaches typically assume a linear torque-deflection relation and access to a motor-torque interface.

Besides classical adaptive controllers, neural adaptive controllers have also been proposed for flexible-joint robots \cite{Chen2020}. In particular, \cite{Chen2020} addresses motor-position controlled systems with uncertain inertial parameters. However, the joints are assumed to have a constant, known stiffness and an ideal motor-position controller.

\subsection{Contribution}
This work is motivated by improving model-based control of motor-position controlled flexible joint robots with uncertain, nonlinear stiffness. In particular: 
\begin{itemize}
	\item We propose an adaptive control framework (cf. Fig.~\ref{fig:intro}) adjusting an estimate of the uncertain, nonlinear torque-deflection map online. The design applies the Lyapunov-based methodology of \cite{Slotine1987}. Different to \cite{Slotine1987}, our approach relies on a control-input-dependent regressor and an implicit control law to handle the uncertain, nonlinear joint stiffness.
	\item We formally analyze robustness of the approach against errors in the inner-loop motor position controller.
	\item We validate the approach experimentally on a flexible joint with strongly nonlinear torque-deflection characteristic, obtaining improved link-side trajectory tracking under adaptation.
\end{itemize}

The remainder of this paper is structured as follows. We introduce the target system dynamics and problem formulation in Sec.~\ref{sec:problem}. The adaptive controller is presented in Sec.~\ref{sec:controller} followed by a robustness analysis against errors induced by the inner-loop motor-position controller in Sec.~\ref{sec:robustness}. We experimentally validate the approach in Sec.~\ref{sec:experiments}. Finally, Sec.~\ref{sec:conclusion} concludes the work.

\section{System Dynamics and Problem Statement}\label{sec:problem}
Consider the dynamics of a robotic system with $ n $ flexible joints modeled as 
\begin{subequations}\label{eq:RD}
	\begin{align}
		\M(\q)\ddq + \T C(\q,\dq)\dq + \T g(\q) &= \springtorque(\qm- \ql)\\
		\T B \ddqm + \springtorque(\qm- \ql) = \T\tau_m.
	\end{align}
\end{subequations}
Therein, $\q \in \Real^n $ and $\qm \in \Real^n $ denote the vectors of link and motor positions, respectively. The link-side inertia matrix $\M(\q) \in \Real^{n\times n} $ is positive definite and is assumed to have bounded eigenvalues. The Coriolis and centrifugal matrix $\T C(\q, \dq) \in \Real^{n\times n}$ fulfills the skew symmetry of $ (\smash{\dot{\M}} -2\T C) $. The motor inertia matrix and the vector of applied motor torques are denoted as $\T{B}\in\smash{\Real^{n\times n}}$ and $ \T\tau_m \in \smash{\Real^n} $, respectively. The motor side and the link side are elastically coupled via the elastic torque vector $\springtorque(\qm- \q) \in \Real^n $. It is assumed to be derived from an elastic potential and thus to be monotonic and Lipschitz continuous in the deflection $\T\theta - \q$.

\subsection{Motor Position Controlled Flexible Joint Robots}
We consider systems \eqref{eq:RD} under motor position control. More precisely, $\T \tau_m $ is provided by an inner-loop controller such that the motor position $ \T\theta $ follows a desired motor position trajectory $ \T\theta_d(t) \in \Real^n$. The link-side dynamics then reduce to
\begin{equation}\label{eq:RD_reduced}
	\M\ddq + \T C\dq + \T g = \springtorque(\T\theta_d - \ql + \T e_\theta ).
\end{equation}
Therein, $\T e_\theta = \T\theta -\T\theta_d \in \Real^n $ denotes the motor-side tracking error. We consider the desired motor angle as our control input, i.e., $ \ctrlinput=\T\theta_d $. 

\subsection{Link-Side Control}\label{sec:Kurzschluss}
If the elastic torque relation $\springtorque(\qm-\q)$ is known and if the motor position error remains sufficiently small such that it can be neglected, a controller for \eqref{eq:RD_reduced} can be implemented by solving
\begin{equation}\label{eq:kurzschluss}
	\T 0 = \springtorque(\ctrlinput - \q) - \T\tau_d
\end{equation}
for the control input $\ctrlinput$. This results in the closed-loop link-side dynamics
\begin{equation}
	\M\ddq + \T C\dq + \T g = \T\tau_d
\end{equation}
corresponding to a rigid-body model driven by the virtual control input $\T\tau_d$, which can be chosen using standard rigid-body robot control techniques. Consequently, the controller design for a motor-position-controlled flexible-joint robot can then be reduced to the one for a conventional rigid-body robot.

\subsection{Problem Statement}
The objective of this work is to develop a model-based adaptive control scheme for the system \eqref{eq:RD_reduced}, that achieves tracking of a desired link-side trajectory ${\q_d(t) \in \Real^n}$ despite uncertainties in the nonlinear torque-deflection model and adapts an estimate of the model online.

\section{Adaptive Control Approach}\label{sec:controller}
The approach presented in this work builds upon the classical adaptive robot controller \cite{Slotine1987} for torque-controlled rigid robots with uncertain inertial parameters. Utilizing the associated Lyapunov-based design methodology, we derive an adaptive controller for the considered class of motor-position-controlled flexible-joint robots with known inertial parameters but uncertain joint stiffness.
For this purpose, the uncertain torque-deflection relation $\springtorque(\qm-\q)$ is expressed as the product of a known, control-input-dependent regressor matrix and an uncertain parameter vector, as detailed in Sec.~\ref{sec:regressor}. Based on this representation, an implicit control law and a parameter adaptation law are formulated in Sec.~\ref{sec:ctrl law} and Sec.~\ref{sec:adaptation}.

\subsection{Assumption on the Motor-Position Controller}
For the derivation of the adaptive control law we will impose the following assumption.
\begin{assumption}\label{a:error}
	The motor position controller operates on a sufficiently fast time scale such that the resulting motor position error $\T e_\theta$ remains negligible for the purposes of the subsequent analysis. In particular, we consider the case $\T e_\theta = \T 0$.
\end{assumption}
We will discard this idealization in Sec.~\ref{sec:robustness} analyzing robustness of the proposed approach.

\subsection{Regressor and Parameter Vector}\label{sec:regressor}
Given Assumption~\ref{a:error}, it holds $ {\springtorque(\T\theta - \q)} = {\springtorque(\T u - \q)} $.
\begin{assumption}\label{a:regressor}
	The torque-deflection relation $\springtorque(\ctrlinput-\q)$ can be decomposed as
	\begin{equation}\label{eq:regressor}
		\springtorque(\ctrlinput-\q) \;=\; \T Y(\ctrlinput-\q)\,\T k,
	\end{equation}
	where $\T Y(\ctrlinput-\q)\in\Real^{n\times p}$ is a regressor matrix containing a finite number of known basis functions of the deflection, and $\T k \in\Real^p$ is an uncertain constant parameter vector.
\end{assumption}

In practice, the required set of basis functions is typically not known a priori. 
However, the mechanical design and calibration measurements usually provide insight into the qualitative structure of $\springtorque(\ctrlinput-\q)$. 
This knowledge allows the selection of a sufficiently rich set of basis functions in $\T Y(\ctrlinput-\q)$ such that the resulting parameterization approximates the true torque-deflection relation over the relevant operating range. 

Given the uncertainty in $ \T k $, only an estimate $ \khat \in \Real^p $ is considered available, in the following. The estimated torque-deflection relation evaluates as 
\begin{equation}\label{eq:springtorque estimated}
	\hat{\springtorque}(\ctrlinput-\q) = \T Y (\ctrlinput-\q) \, \khat. 
\end{equation}
We will utilize this relation to formulate a link-side tracking controller and then provide a parameter adaptation law for $ \khat $, such that the tracking errors converge to zero.

\subsection{Control Law}\label{sec:ctrl law}

Following \cite{Slotine1987}, we aim at injecting the torque
\begin{equation}\label{eq:slotine li}
	\T\tau_d = \M(\q)\ddq_r + \T{C}(\q,\dq)\dqr + \T g(\q) -\T{K}_s(t)\T{s}
\end{equation}
on the link-side, i.e., a tracking controller based on the sliding-variable $ \T s \in \Real^n $. The sliding variable is defined as 
\begin{equation}\label{eq:s}
	\T{s } = \dqerr + \T{K}_p \qerr = \dq - \dq_r
\end{equation}
with $ \dq_r = \dq_d - \T{K}_p \qerr $ and constant, positive definite gain matrix $ {\T{K}_p \in\Real^{n\times n}}$. By construction, convergence of $\T s$ to zero implies convergence of the tracking errors $\qerr = \q - \q_d$ and $\dqerr = \dq - \dq_d$ to zero. 
The gain matrix $ \Ks(t)\in\Real^{n \times n} $ in \eqref{eq:slotine li} is continuous, positive definite with uniformly bounded eigenvalues.

If the true parameter vector $\T k$ from \eqref{eq:regressor} were known, the motor input could be computed directly from \eqref{eq:kurzschluss} in order to inject the desired torque \eqref{eq:slotine li} exactly. With only the estimate~$ \khat $ available, we instead choose $ \ctrlinput $ such that the following is satisfied
\begin{equation}\label{eq:control law}
	\T 0 = \T Y(\ctrlinput- \q)\khat \, - \, \underbrace{ ( \M\ddq_r + \T{C}\dqr + \T g -\Ks\T{s})}_{\T\tau_d} .
\end{equation}
Thus, the motor position is chosen such that the \emph{estimated} spring torque balances the desired torque $\T \tau_d$. Equation \eqref{eq:control law} constitutes an implicit control law. Existence of an analytical solution for $\ctrlinput$ depends on the utilized basis functions in the regressor. Two common special cases admitting an explicit formulation of the control law are addressed in the following. If no analytical solution exists, \eqref{eq:control law} must be solved numerically for the control input. 

\subsubsection{Case 1: Linear Regressor}
For joints with linear torque-deflection relation, where the only uncertainty is the joint stiffness, the regressor $ {\T Y(\ctrlinput- \q) = \textup{diag}(\ctrlinput - \q) \in \Real^{n\times n}} $ only contains linear functions. Equation \eqref{eq:control law} can then be solved for $ \ctrlinput  $ as
\begin{equation}\label{eq:linear}
	\ctrlinput = \q + \inv{\left(\textup{diag}\left(\khat\right)\right)}\T{\tau_d},
\end{equation}
i.e., there exists an explicit control law. To ensure well-defined torque generation, the entries of $\khat$ must remain nonzero.

\subsubsection{Case 2: Special Cubic Regressor}\label{sec:cubic}
Consider a cubic torque-deflection relation of the $i$-th joint of the form ${ \varphi_i(\theta_i - q_i) = k_{3,i}(\theta_i - q_i)^3 + k_{1,i}(\theta_i - q_i)} $ with linear and cubic stiffness coefficients $k_{1,i}, k_{3,i} > 0$, where all uncertainty is assumed to be contained in the coefficients. The mapping $ \varphi_i(\theta_i - q_i) $ is strictly monotonically increasing. Hence, the relation is globally invertible.

For this special case, the implicit control law~\eqref{eq:control law} reduces componentwise to the cubic equation
\begin{equation}
	\begin{bmatrix}
		(u_{i} - q_i)^3 & (u_{i} - q_i)
	\end{bmatrix}
\begin{bmatrix}
	\hat{k}_{3,i} \\ \hat{k}_{1,i}
\end{bmatrix}
	= \tau_{d,i}.
\end{equation}

Using Cardano's formula, the solution $ u_{i}$ can be expressed analytically for arbitrary desired torques $ \tau_{d,i} \in \Real$ as
\begin{equation}\label{eq:cubic_solution}
	u_{i}
	= q_i
	+ \sqrt[3]{\frac{\tau_{d,i}}{2\hat{k}_{3,i}} + a_i}
	+ \sqrt[3]{\frac{\tau_{d,i}}{2\hat{k}_{3,i}} - a_i},
\end{equation}
with
\begin{equation}
	a_i
	= \sqrt{
		\frac{\tau_{d,i}^2}{4\hat{k}_{3,i}^2}
		+ \frac{\hat{k}_{1,i}^3}{27\hat{k}_{3,i}^3}
	}.
\end{equation}
as long as the estimates $ \hat{k}_{3,i} $ and $ \hat{k}_{1,i} $ remain nonzero.

\subsection{Adaptation Law}\label{sec:adaptation}
Substituting \eqref{eq:control law} in the robot dynamics \eqref{eq:RD_reduced} with Assumptions~\ref{a:error} and~\ref{a:regressor}
yields the closed-loop dynamics
\begin{align}
	\M\ddq + \T C\dq + \T g &= \T{\tau}_d + \T{Y}(\ctrlinput- \q)\kerr,\\
	\label{eq:closed loop}
	\M\dot{\T s}  + (\T{C} +\Ks)\T{s} &=  \T{Y}(\ctrlinput- \q)\kerr.
\end{align}
The term $ \T{Y}(\ctrlinput- \q)\kerr $ acts as a disturbance decreasing with the parameter estimation error $ \kerr = \T k -\khat $. We will formulate an adaptation law, such that $\T{s}$ converges to zero despite this disturbance. 
We employ the Lyapunov function candidate
\begin{equation}\label{eq:S}
	{S}(\T s, \kerr, t) = \underbrace{\frac{1}{2}\T{s}^T\M(\q)\T{s}}_{S_s(\T s, t)} + \underbrace{\frac{1}{2}\kerr^T\T{A}\kerr}_{S_k(\kerr)}.
\end{equation}
which is structurally identical to the one in \cite{Slotine1987}. It consists of a time-dependent tracking contribution $S_s(\T s, t)$ and a parameter error contribution $S_k(\kerr)$ with diagonal, positive definite matrix $ \T A \in \Real^{p\times p} $.

As the inertia matrix $\M(\q)$ is symmetric and positive definite with uniformly bounded eigenvalues, ${S}$ is bounded from below and above by positive definite functions in $(\T s,\kerr)$.

The time derivative of $ S $ along the trajectories of the closed-loop system evaluates as
\begin{align}
	\dot{{S}} & = -\T{s}^T\Ks \T{s} \, + \, \left(\dkerr^T\T{A} \,+\, \T{s}^T\T{Y}(\ctrlinput-\q)\right)\kerr
\end{align}
due to skew-symmetry of $\dot{\M} -2\T{C}$. It becomes negative-semi-definite for the parameter adaptation law
\begin{equation}\label{eq:adaptation law}
	 \dot{\hat{\T{k}}} = \inv{\T{A}}\T{Y}^T\T{s},
\end{equation}
which renders the equilibrium $\begin{bmatrix}
	\T{s}^T & \kerr^T
\end{bmatrix} = \T 0 $ uniformly globally stable.

To analyze convergence properties, Barbalat's Lemma can be utilized equivalently as in \cite{Slotine1987}. In short, the boundedness of
$\ddot{S} = -\T s^T\dot{\Ks}\T s + 2\T{s}^T\Ks\dot{\T{s}}$ implies that $\dot{S}$, which is known to be bounded, converges to zero for $t\rightarrow\infty$. Thus $\T{s}$ globally uniformly converges to zero for $t\rightarrow\infty$. By \eqref{eq:s}, this implies global uniform convergence of $ \begin{bmatrix}
	\qerr^T & \dqerr^T
\end{bmatrix} \rightarrow \T 0 $ for $t\rightarrow\infty$.

\subsection{Summary}
We summarize the previous results as follows.
\begin{theorem}
	Consider the motor-position controlled flexible joint robot \eqref{eq:RD_reduced} with uncertain torque-deflection relation $ \springtorque(\T\theta - \T q) = \T Y(\T\theta - \T q)\,\T k $ under Assumptions~\ref{a:error}~and~\ref{a:regressor}. Under the control law \eqref{eq:control law} and adaptation law \eqref{eq:adaptation law} global uniform convergence of $\begin{bmatrix}
		\qerr^T & \dqerr^T
	\end{bmatrix} \rightarrow \T 0$ is achieved for $ t \rightarrow \infty$. A parameter estimate $\khat$ is adapted online such that the parameter estimation error remains bounded.
\end{theorem}
Note that convergence of $\khat$ is not guaranteed. Figure~\ref{fig:intro} shows the overall structure of the controller.

\section{Robustness against Errors Induced by the Motor Position Control}\label{sec:robustness}
Let us now evaluate robustness of the proposed approach against errors induced by the inner-loop controller. Consider \eqref{eq:RD_reduced}. It holds \begin{align}\label{key}
	\springtorque(\T\theta - \q) &=  \springtorque\left(\T\theta_d - \q + \etheta\right)\\
	\label{eq: k etheta}
	&=  \springtorque\left(\T\theta_d - \q \right) + \taue,
\end{align}
with $ \taue \in \Real^n $ and $ \abs{\taue} \leq \klipschitzmax \, \abs{\etheta}$ due to the monotonicity and Lipschitz continuity of the spring torque. Here, $ \klipschitzmax > 0 $ denotes the Lipschitz constant of the joint with the largest maximum stiffness. Applying the adaptive controller \eqref{eq:control law} and \eqref{eq:adaptation law} then results in the closed-loop system
\begin{subequations}\label{eq:closed loop disturbed}
	\begin{align}
		\label{eq:a}
		\M\dot{\T s}  + (\T{C} +\Ks)\T{s} &=  \T{Y}(\ctrlinput- \q)\kerr + \taue \\
		\label{eq:b}
		\dkerr &= -\inv{\T{A}}\T{Y}^T\T{s}.
	\end{align}
\end{subequations}
The error dynamics \eqref{eq:a} in $ \T s $ now includes an additional disturbance term $ \T{\tau_e} $ as an input, which depends on the motor position error $ \etheta $.

\begin{theorem}Consider the system \eqref{eq:RD_reduced} with motor position error $ \etheta(t) $ under the control law \eqref{eq:control law} and adaptation law \eqref{eq:adaptation law}. The following holds:
	\begin{enumerate}
		\item The closed-loop system is finite-gain $ \Ltwo $ stable with respect to the input $ \etheta $ and the output $ \T s $. 
		\item If $ \etheta \in \Ltwo $, $ \T s $ and $ \kerr $ are uniformly bounded. Further, $ \T s \rightarrow \T 0 $ for $ t \rightarrow \infty $ and thus $ \begin{bmatrix}
			\qerr^T & \dqerr^T
		\end{bmatrix} \rightarrow \T 0 $ for $ t \rightarrow \infty $.
	\end{enumerate}
\end{theorem}

\begin{proof}
Consider the function $ S(\T s,\kerr,t) $ from \eqref{eq:S}. The time derivative along the trajectories of the closed-loop system \eqref{eq:closed loop disturbed} is then obtained as
\begin{align}\label{eq:dS_disturbed}
	\dot{S} &= - \T{s}^T \Ks \T{s} + \T s^T \taue\\
	\label{eq: dissipation}
	&\leq - \underbrace{(1 - \epsilon)\ksmin}_{\coloneq c_1} \, \T{s}^T\T s + \underbrace{\inv{(4\epsilon \ksmin)}}_{\coloneq c_2} \, \taue^T\taue
\end{align}
where $ \ksmin > 0 $ is the smallest eigenvalue of $ \Ks $ and $ \epsilon$ is an arbitrary constant, such that it holds $ 0 < \epsilon < 1 $.\footnote{Note that \eqref{eq: dissipation} holds due to Young's inequality.} The positive constants $ c_1$ and $ c_2 $ are defined by this relation. Using this as a starting point, we prove the statements 1 and 2 consecutively.

\subsubsection*{Statement 1)}
Let us integrate $ \eqref{eq: dissipation} $ from $ t = 0 $ to $ t = T $. We obtain
\begin{align}\label{eq: integrated S}
	S_T - S_0 \leq - c_1 \int_{0}^{T}\T{s}^T\T s \, dt + c_2 \int_{0}^{T}\taue^T \taue \, dt
\end{align}
with $ S_T =S(\T s(T), \kerr(T), T) $ and $ S_0 = S(\T s(0), \kerr(0), 0) $. With $ S_T \geq 0 $, this immediately yields

\begin{align}\label{key}
	\sqrt{\int_{0}^{T} \T s^T \T s \, dt} &\leq \, \gamma_1 \,  \sqrt{\int_{0}^{T} \taue^T \taue \, dt} + \beta \\
	&\leq \, \gamma_2 \,  \sqrt{\int_{0}^{T} \etheta ^T \etheta \, dt} + \beta
\end{align}
with $ \gamma_1 =  \sqrt{c_2\inv{c_1}}$, $ {\beta = \sqrt{S_0\inv{c_1}}} $  and $ \gamma_2 = \klipschitzmax\gamma_1 $. Considering $ T \rightarrow \infty $, this implies that the closed-loop system with input $ \etheta $ and output $ \T s $ is finite-gain $ \Ltwo $ stable (c.f., e.g. \cite[Definition 5.1]{Khalil2002}).

\subsubsection*{Statement 2)}
Consider again \eqref{eq: integrated S}. For $ \etheta \in \Ltwo $ it holds $ \taue \in \Ltwo $. Further, we just showed that then $ \T s  \in \Ltwo $. Thus, the right hand side of \eqref{eq: integrated S} is bounded $ \forall t $. This implies uniform boundedness of $ S_T $, which, in turn, yields uniform boundedness of $ \T s $ and $ \kerr $.

Consider the function $ {f(t) = \int_{0}^{t} \T{s}^T\T s \, dr} $. Given $ \etheta \in \Ltwo $ it holds $ \T s \in \Ltwo $ as it has been shown in the proof of Statement 1. As $ f(t) $ is thus upper bounded and as the integrand is monotonically increasing, $ f(t) $ approaches a finite limit for $ t \rightarrow \infty $.

Now consider $ \ddot{f}(t) = 2 \T s^T\dot{\T{s}} $. As $ \T s $ and $ \kerr $ have been shown to be bounded, we can conclude boundedness of $ \dot{\T s} $ from \eqref{eq:a}. Thus, $ \ddot{f}(t) $ is bounded. This implies that $ \dot{f}(t) = \T s^T\T s $ is uniformly continuous.

By Barbalat's Lemma, we can now conclude $ \dot{f}(t)\rightarrow 0 $ for $ t \rightarrow \infty $. Thus, $ \T{s}\rightarrow\T 0 $ and $ \begin{bmatrix}
	\qerr^T & \dqerr^T
\end{bmatrix} \rightarrow \T 0$.
	
\end{proof}

\section{Experiments}\label{sec:experiments}
To validate the performance of the approach under distinct torque-deflection characteristics, we perform trajectory tracking experiments with a variable stiffness actuator varying the physical stiffness during the motion.

\subsection{Setup}
The utilized setup is depicted in Fig.~\ref{fig:setup}. It consists of the motor-position controlled VSA \cite{Catalano2011} rotating a weight of $\SI{0.5}{\kilogram}$ attached on a lever in the horizontal plane.
\begin{figure}[tb]
	\vspace{2mm}
	\centering
	\includegraphics[width=0.46\textwidth]{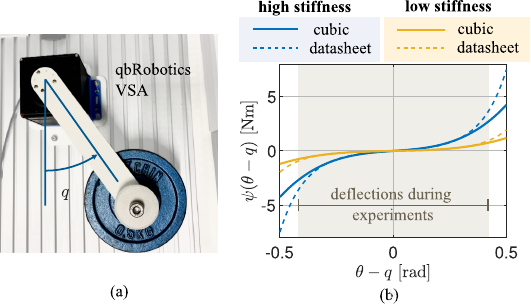}
	\caption{Experimental setup consisting of one VSA by qbRobotics rotating a $\SI{0.5}{\kilogram}$ weight in a horizontal plane (a). The utilized cubic torque-deflection models for the high and low stiffness settings approximately match the datasheet models for relevant deflections $\vert \theta - q \vert \leq \SI{0.4}{\radian}$ (b).}
	\label{fig:setup}
\end{figure}

\subsubsection{Dynamics and Nominal Torque-Deflection Model}
The considered actuator contains two servo motors with motor positions $ \T \theta_1 $ and $ \T\theta_2 $. Stiffness and position adjustments are decoupled \cite{Catalano2011}. Via increasing the relative motor position ${\Delta \T\theta = \num{0.5}(\T\theta_2 - \T\theta_1)} $, the stiffness setting can be increased, whereas the equilibrium position is determined as ${\T\theta = \num{0.5}(\T\theta_1 + \T\theta_2)}$. 
Given a fixed $ \Delta\T\theta $, the actuator can thus be treated as a series elastic actuator \cite{Lentini2019,Moyron2025}. The link-side system dynamics can then be expressed in terms of $\T{\theta}$ in the form \eqref{eq:RD_reduced} as 
\begin{align}
	\M \ddq &= \springtorque(\T\theta - \q)\\
	& = \springtorque(\ctrlinput - \q) + \taue
\end{align}
with $\M = \SI{0.0129}{\newton\metre\second^2/\radian}$ and control input $\T{u}$.

In the following, we operate the actuator at two distinct relative motor positions. We will refer to the setting with $ {\Delta\T\theta = \SI{0.5}{\radian}} $ as the \emph{high stiffness} preset and $ {\Delta\T\theta = \SI{0.35}{\radian}} $ as the \emph{low stiffness} preset. According to the datasheet and \cite{Lentini2019}, the corresponding nonlinear torque-deflection relations with deflection $\T{\delta} =  \T\theta - \q$ are obtained as 
\begin{subequations}\label{eq:cube data}
	\begin{align}\label{eq:cube data high}
	\springtorque_{high}(\T{\delta}) &= 0.166\, \textup{sinh}\left(8.999\, \T{\delta} \right),\\
	\label{eq:cube data low}
	\springtorque_{low}(\T{\delta}) &= 0.043\, \textup{sinh}\left(8.999\, \T{\delta} \right),
\end{align}
\end{subequations}
respectively.

\subsubsection{Regressor and Parameter Vector}
Considering \eqref{eq:cube data}, a hyperbolic sine function could be a suitable choice for a basis function of the regressor. Moreover, as depicted in Fig.~\ref{fig:setup}(b), the functions \eqref{eq:cube data} can be approximated by the cubic function
\begin{align}\label{eq:cube cubic}
	\springtorque_{cubic} (\T{\delta}) &= k_3 \T{\delta}^3 + k_1 \T{\delta}\\
	&= \begin{bmatrix}
		\T{\delta}^3 & \T{\delta}
	\end{bmatrix} \begin{bmatrix}
	k_3 \\ k_1
	\end{bmatrix} = \T{Y}(\T\delta)\T k
\end{align}
with parameters $ {\T k_{high} = \begin{bmatrix}
	28 & 1.5
\end{bmatrix}^T} $ for \eqref{eq:cube data high} and $ {\T k_{low} = \begin{bmatrix}
8 & 0.38
\end{bmatrix}^T} $ for \eqref{eq:cube data low} in a range of deflections ${\T\delta \leq \SI{0.4}{\radian}}$. 

We will utilize the cubic regressor  $\T Y(\ctrlinput-\q) $ from \eqref{eq:cube cubic} with $\T\delta = \ctrlinput -\q$ for the formulation of the adaptive controller. This choice is guided by the datasheet relation and allows for an explicit formulation of the control law while the behavior for large and small deflections can be adjusted distinctly via the cubic and linear parameter.

 \subsubsection{Desired Trajectory}
For the experiments, we aim at tracking a desired link-side trajectory 
\begin{equation}\label{eq:traj}
	\q_d(t) = 0.5 \pi \left( \sin\left(\frac{8\pi}{T} t\right) + \sin\left(\frac{2\pi}{T} t\right) \right),
\end{equation}
which is periodic with $T = \SI{6.5}{\second}$. The actuator's stiffness setting is held constant for two consecutive periods before being switched between the defined \emph{high} and \emph{low} settings. This alternating pattern is maintained over multiple cycles throughout the experiment.

\begin{table}[tb]
	\vspace{2mm}
	\centering
	\includegraphics[]{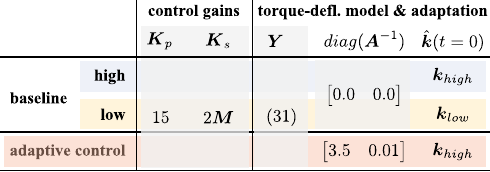}
	\caption{Gains and Torque-deflection-model used for the implementation of the baseline and proposed approaches.}
	\label{tab:gains}
\end{table}

\subsubsection{Control Approaches}
We compare the performance of the proposed \emph{adaptive controller} using an initial parameter estimate $ \khat_{0} = \T k_{high} $ against two baseline approaches, which we will refer to as \emph{baseline - high} and \emph{baseline - low}. These use the same control law \eqref{eq:control law}, but the constant parameter estimates $ \khat = \T k_{high} = const. $ and $ \khat = \T k_{low} = const. $, respectively, i.e., we deactivate the parameter adaptation. 
To ensure physically meaningful estimates of the torque-deflection models and to prevent numerical issues due to \eqref{eq:cubic_solution}, we implement lower thresholds for the parameter estimates, i.e. $ \hat k_1, \, \hat k_3 \geq 10^{-3} $.

\begin{figure*}[htb]
	\vspace{2mm}
	\centering
	\includegraphics[]{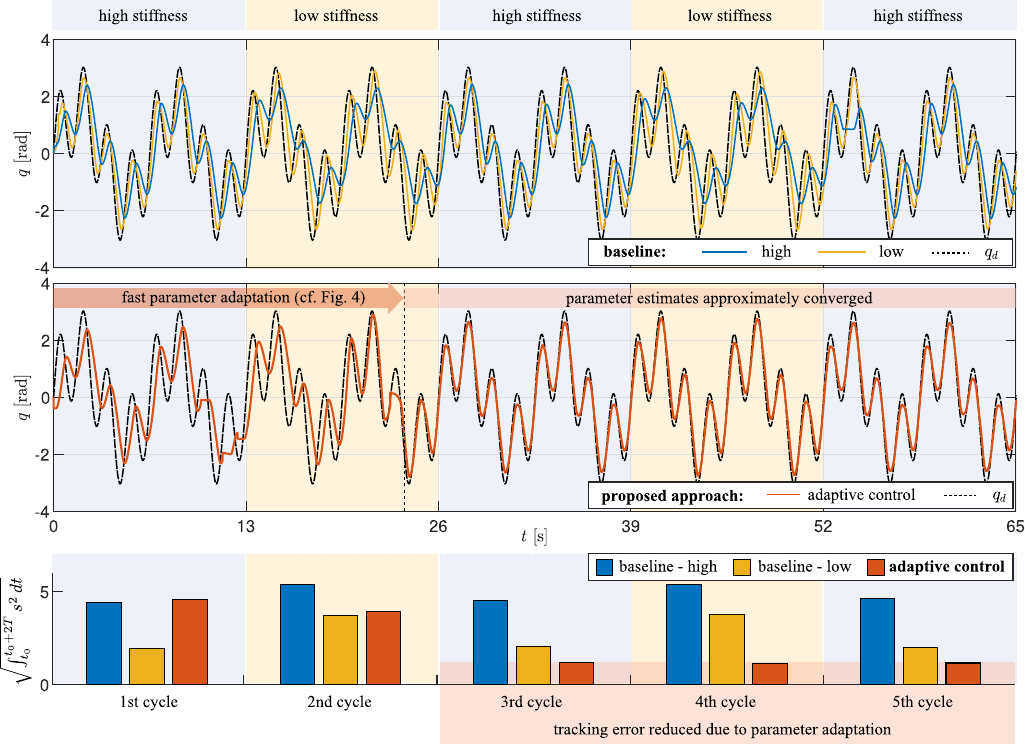}
	\caption{Tracking performance of baseline approaches (cf. Table~\ref{tab:gains}) utilizing constant torque-deflection models (top) versus our proposed adaptive controller (middle). The norm of $\T s$ integrated over a cycle reduces under the parameter adaptation (bottom), such that the adaptive controller eventually outperforms the baseline approaches both for high and low stiffness presets of the VSA.}
	\label{fig:tracking}
\end{figure*}

We provide an overview of utilized gains and initial parameter estimates in Table~\ref{tab:gains}.

The controllers are implemented at a control frequency of $\SI{200}{\hertz}$.

\subsection{Results}
Figure~\ref{fig:tracking} shows the tracking results over five stiffness cycles. We provide $ \q(t) $ under the baseline approaches without adaptation (top), $ \q(t) $ under the proposed adaptive controller (middle) and the normed error $ \sqrt{\int_{t_0}^{t_0 + 2T} \T s^T\T s \, dt}$ per cycle with cycle starting time $ t_0 $.

As the adaptive controller (red) has been initialized with $ \khat(t=0) = \T k_{high} $, its performance is comparable to the baseline approach using the high stiffness model (blue) during the first cycle. However, then the tracking of $ \q_d(t) $ improves. Starting from the third cycle, the adaptive controller achieves the least normed tracking error in both the cycles with high and low stiffness setting.

The evolution of the parameter estimates is depicted in Fig.~\ref{fig:params}. Both the cubic and the linear parameter estimates reduce during the \emph{high} stiffness setting of the first cycle. With the second cycle, the adaptation rate further decreases to account for the reduced stiffness. Notably, the cubic estimate reaches its lower threshold at almost zero during this cycle. Conversely, the linear parameter stabilizes at $ \hat{k}_1\approx 0.5 $, which is larger than the linear contribution to $\T{k}_{low}$ obtained from the datasheet.

\begin{figure}[tb]
	\vspace{1mm}
	\centering
	\includegraphics[width=0.46\textwidth]{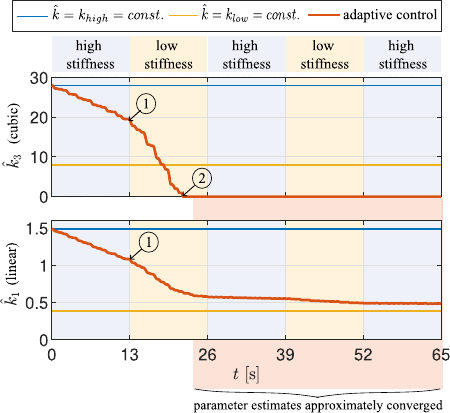}
	\caption{Evolution of the parameter estimates $ \hat{k}_3 $ and $ \hat{k}_1 $ during the adaptation as compared to the datasheet values. The absolute value of the adaptation rate clearly increases as the VSA's stiffness preset is reduced \textcircled{1}. The cubic parameter estimate runs into its lower limit \textcircled{2} whereas the linear parameter estimate converges to a value above the datasheet value for the low stiffness preset.}
	\label{fig:params}
\end{figure}

It may contradict the expectation that the parameter estimates do not increase in the third and fifth cycle, after the stiffness setting is switched from low to high. However, the adaptive controller does not guarantee convergence of the parameter estimate to its real value. It only adjusts the estimate, if there is a tracking error. Notably, the tracking performance remains consistent starting from the third cycle even without parameter adjustment. This may be due to the fact that the estimated, low, constant stiffness  can be understood as high gain on the desired link-side torque (cf. \eqref{eq:linear}). 
Note that this interpretation aligns with the observation that the \textit{baseline - low} controller (yellow) is more performant than the \textit{baseline~- high} one (blue).

\section{Conclusion and Outlook}\label{sec:conclusion}
This work addressed adaptive control of motor-position-controlled flexible-joint robots with uncertain, nonlinear joint stiffness. We designed an implicit control law and an adaptation law for the case of an ideal motor-position interface proving convergence of tracking errors to zero. Further, we analyzed robustness of the approach with respect to motor-side control errors. 

Experimental results confirm that the proposed approach can enhance the tracking performance of a motor-position-controlled flexible joint with an uncertain torque-deflection relation. The approach was able to react to the stiffness changes introduced via a variable stiffness mechanism and outperformed controllers utilizing a constant torque-deflection model. This motivates applying the approach in robotic joints, whose torque-deflection characteristics are expected to strongly depend on operating conditions. This could, e.g., be promising for systems involving low-cost elastic elements such as the springs fabricated from PLA introduced in \cite{Hoppner2025a}.

Future work could develop adaptation laws that ensure physically meaningful estimates of the stiffness characteristics. To improve parameter estimation, it seems also promising to utilize information from the torque prediction error as in \cite{Slotine1989}. 

Finally, extensions to further classes of robots with joint flexibility, such as antagonistically driven VSA or twisted string actuators, seem promising.

\bibliographystyle{IEEEtran}
\bibliography{AdaptiveControl}

@inproceedings{Catalano2011,
  title = {{{VSA-CubeBot}}: {{A}} Modular Variable Stiffness Platform for Multiple Degrees of Freedom Robots},
  shorttitle = {{{VSA-CubeBot}}},
  booktitle = {2011 {{IEEE International Conference}} on {{Robotics}} and {{Automation}}},
  author = {Catalano, Manuel G. and Grioli, Giorgio and Garabini, Manolo and Bonomo, Fabio and Mancini, Michele and Tsagarakis, Nikolaos and Bicchi, Antonio},
  year = 2011,
  month = may,
  pages = {5090--5095},
  issn = {1050-4729},
  doi = {10.1109/ICRA.2011.5980457},
  urldate = {2026-02-05}
}

@inproceedings{Chen2020,
  title = {Adaptive {{Neural Trajectory Tracking Control}} for {{Flexible-Joint Robots}} with {{Online Learning}}},
  booktitle = {2020 {{IEEE International Conference}} on {{Robotics}} and {{Automation}} ({{ICRA}})},
  author = {Chen, Shuyang and Wen, John T.},
  year = 2020,
  month = may,
  pages = {2358--2364},
  issn = {2577-087X},
  doi = {10.1109/ICRA40945.2020.9197051},
  urldate = {2026-02-09}
}

@article{DellaSantina2017,
  title = {Controlling {{Soft Robots}}: {{Balancing Feedback}} and {{Feedforward Elements}}},
  shorttitle = {Controlling {{Soft Robots}}},
  author = {Della Santina, Cosimo and Bianchi, Matteo and Grioli, Giorgio and Angelini, Franco and Catalano, Manuel and Garabini, Manolo and Bicchi, Antonio},
  year = 2017,
  month = sep,
  journal = {IEEE Robotics \& Automation Magazine},
  volume = {24},
  number = {3},
  pages = {75--83},
  issn = {1558-223X},
  doi = {10.1109/MRA.2016.2636360},
  urldate = {2026-02-05}
}

@incollection{DeLuca2008,
  title = {Robots with {{Flexible Elements}}},
  booktitle = {Springer {{Handbook}} of {{Robotics}}},
  author = {De Luca, Alessandro and Book, Wayne},
  editor = {Siciliano, Bruno and Khatib, Oussama},
  year = 2008,
  pages = {287--319},
  publisher = {Springer Berlin Heidelberg},
  address = {Berlin, Heidelberg},
  doi = {10.1007/978-3-540-30301-5_14},
  urldate = {2026-03-02},
  isbn = {978-3-540-23957-4 978-3-540-30301-5},
  langid = {english}
}

@inproceedings{Ghorbel1989,
  title = {Adaptive Control of Flexible Joint Manipulators},
  booktitle = {1989 {{International Conference}} on {{Robotics}} and {{Automation Proceedings}}},
  author = {Ghorbel, F. and Hung, J.Y. and Spong, M.W.},
  year = 1989,
  month = may,
  pages = {1188-1193 vol.2},
  doi = {10.1109/ROBOT.1989.100141},
  urldate = {2026-02-09}
}

@article{Hoppner2025a,
  title = {Variable {{Stiffness Actuation}} via {{3D-Printed Nonlinear Torsional Springs}}},
  author = {H{\"o}ppner, Hannes and Kirner, Annika and G{\"o}ttlich, Joshua and Jakob, Linn{\'e}a and Dietrich, Alexander and Ott, Christian},
  year = 2025,
  month = may,
  journal = {IEEE Robotics and Automation Letters},
  volume = {10},
  number = {5},
  pages = {4324--4331},
  issn = {2377-3766},
  doi = {10.1109/LRA.2025.3549658},
  urldate = {2026-02-05}
}

@book{Khalil2002,
  title = {Nonlinear Systems},
  author = {Khalil, Hassan K.},
  year = 2002,
  edition = {3. ed},
  publisher = {Prentice Hall},
  address = {Upper Saddle River, NJ},
  isbn = {978-0-13-067389-3},
  langid = {english}
}

@inproceedings{Lee2018,
  title = {A {{Natural Adaptive Control Law}} for {{Robot Manipulators}}},
  booktitle = {2018 {{IEEE}}/{{RSJ International Conference}} on {{Intelligent Robots}} and {{Systems}} ({{IROS}})},
  author = {Lee, Taeyoon and Kwon, Jaewoon and Park, Frank C.},
  year = 2018,
  month = oct,
  pages = {1--9},
  publisher = {IEEE},
  address = {Madrid},
  doi = {10.1109/IROS.2018.8593727},
  urldate = {2026-02-15},
  isbn = {978-1-5386-8094-0}
}

@article{Lentini2019,
  title = {Alter-{{Ego}}: {{A Mobile Robot With}} a {{Functionally Anthropomorphic Upper Body Designed}} for {{Physical Interaction}}},
  shorttitle = {Alter-{{Ego}}},
  author = {Lentini, Gianluca and Settimi, Alessandro and Caporale, Danilo and Garabini, Manolo and Grioli, Giorgio and Pallottino, Lucia and Catalano, Manuel G. and Bicchi, Antonio},
  year = 2019,
  month = dec,
  journal = {IEEE Robotics \& Automation Magazine},
  volume = {26},
  number = {4},
  pages = {94--107},
  issn = {1558-223X},
  doi = {10.1109/MRA.2019.2943846},
  urldate = {2026-02-05}
}

@inproceedings{Lozano-Leal1991,
  title = {Adaptive {{Control}} of {{Robot Manipulators}} with {{Flexible Joints}}},
  booktitle = {1991 {{American Control Conference}}},
  author = {{Lozano-Leal}, Rogelio and Brogliato, Bernard},
  year = 1991,
  month = jun,
  pages = {938--943},
  doi = {10.23919/ACC.1991.4791516},
  urldate = {2026-02-05}
}

@article{Moyron2025,
  title = {Elastic {{Structure Preserving Control}} for {{Flexible Joint Robots With Position-Controlled Actuators}}},
  author = {Moyr{\'o}n, Jer{\'o}nimo and Ott, Christian and Kirner, Annika and {Moreno-Valenzuela}, Javier},
  year = 2025,
  month = sep,
  journal = {IEEE Transactions on Control Systems Technology},
  volume = {33},
  number = {5},
  pages = {1810--1819},
  issn = {1558-0865},
  doi = {10.1109/TCST.2025.3562024},
  urldate = {2026-02-05}
}

@inproceedings{Pierallini2020,
  title = {Trajectory {{Tracking}} of a {{One-Link Flexible Arm}} via {{Iterative Learning Control}}},
  booktitle = {2020 {{IEEE}}/{{RSJ International Conference}} on {{Intelligent Robots}} and {{Systems}} ({{IROS}})},
  author = {Pierallini, Michele and Angelini, Franco and Mengacci, Riccardo and Palleschi, Alessandro and Bicchi, Antonio and Garabini, Manolo},
  year = 2020,
  month = oct,
  pages = {7579--7586},
  issn = {2153-0866},
  doi = {10.1109/IROS45743.2020.9341215},
  urldate = {2026-02-05}
}

@article{Slotine1987,
  title = {On the {{Adaptive Control}} of {{Robot Manipulators}}},
  author = {Slotine, Jean-Jacques E. and Li, Weiping},
  year = 1987,
  month = sep,
  journal = {The International Journal of Robotics Research},
  volume = {6},
  number = {3},
  pages = {49--59},
  issn = {0278-3649, 1741-3176},
  doi = {10.1177/027836498700600303},
  urldate = {2025-04-02},
  copyright = {https://journals.sagepub.com/page/policies/text-and-data-mining-license},
  langid = {english}
}

@article{Slotine1989,
  title = {Composite Adaptive Control of Robot Manipulators},
  author = {Slotine, Jean-Jacques E. and Li, Weiping},
  year = 1989,
  month = jul,
  journal = {Automatica},
  volume = {25},
  number = {4},
  pages = {509--519},
  issn = {00051098},
  doi = {10.1016/0005-1098(89)90094-0},
  urldate = {2025-04-02},
  copyright = {https://www.elsevier.com/tdm/userlicense/1.0/},
  langid = {english}
}

@article{Spong1987,
  title = {Modeling and {{Control}} of {{Elastic Joint Robots}}},
  author = {Spong, M. W.},
  year = 1987,
  month = dec,
  journal = {Journal of Dynamic Systems, Measurement, and Control},
  volume = {109},
  number = {4},
  pages = {310--318},
  issn = {0022-0434, 1528-9028},
  doi = {10.1115/1.3143860},
  urldate = {2026-03-03},
  langid = {english}
}

@article{Spong1989,
  title = {Adaptive Control of Flexible Joint Manipulators},
  author = {Spong, Mark W.},
  year = 1989,
  month = jul,
  journal = {Systems \& Control Letters},
  volume = {13},
  number = {1},
  pages = {15--21},
  issn = {0167-6911},
  doi = {10.1016/0167-6911(89)90016-9},
  urldate = {2026-02-09}
}

@inproceedings{Tian1995,
  title = {Robust Adaptive Control of Flexible Joint Robots with Joint Torque Feedback},
  booktitle = {Proceedings of 1995 {{IEEE International Conference}} on {{Robotics}} and {{Automation}}},
  author = {Tian, Lin and Goldenberg, A.A.},
  year = 1995,
  month = may,
  volume = {1},
  pages = {1229-1234 vol.1},
  issn = {1050-4729},
  doi = {10.1109/ROBOT.1995.525448},
  urldate = {2026-02-05}
}

@article{Zambella2025a,
  title = {Composite {{Whole-Body Control}} of {{Two-Wheeled Robots}}},
  author = {Zambella, Grazia and Caporale, Danilo and Grioli, Giorgio and Pallottino, Lucia and Bicchi, Antonio},
  year = 2025,
  journal = {IEEE Transactions on Robotics},
  volume = {41},
  pages = {2301--2321},
  issn = {1941-0468},
  doi = {10.1109/TRO.2025.3548494},
  urldate = {2026-03-02}
}

\end{document}